\documentclass[a4paper,12pt]{article}

\usepackage{amsmath,amssymb,amsfonts,amsthm} 
\usepackage{pstricks, graphicx}

\newtheorem{theorem}{Theorem}[section]
\newtheorem{lemma}[theorem]{Lemma}

\newtheorem{corollary}[theorem]{Corollary}

\title{Ordinal Risk-Group Classification}

\author{Yizhar Toren \\ Tel Aviv University, Tel Aviv, Israel \textsf{yizhar.toren@math.tau.ac.il} } %

\begin{document}

%
%
%
%
%
%
%
%
%
%

\maketitle

%

\begin{abstract}
Most classification methods provide either a prediction of class membership or an assessment of class membership probability. In the case of two-group classification the predicted probability can be described as "risk" of belonging to a ``special" class . When the required output is a set of ordinal-risk groups, a discretization of the continuous risk prediction is achieved by two common methods: by constructing a set of models that describe the conditional risk function at specific points (quantile regression) or by dividing the output of an "optimal" classification model into adjacent intervals that correspond to the desired risk groups. By defining a new error measure for the distribution of risk onto intervals we are able to identify lower bounds on the accuracy of these methods, showing sub-optimality both in their distribution of risk and in the efficiency of their resulting partition into intervals. By adding a new form of constraint to the existing maximum likelihood optimization framework and by introducing a penalty function to avoid degenerate solutions, we show how existing methods can be augmented to solve the ordinal risk-group classification problem. We implement our method for generalized linear models (GLM) and show a numeric example using Gaussian logistic regression as a reference.
\end{abstract}

\section{Introduction \label{section_intro}}

\par The classical problem of discriminating between two classes of observations based on a given dataset has been widely discussed in the statistical literature. When only two classes are involved, the question of discrimination is reduced to whether or not a given observation is a member of a ''special" class (where the other class is the default state, for example sick vs. healthy). Some classification methods, such as Fisher's linear discriminant analysis (LDA), make a decisive prediction of class membership while minimizing error in some sense, typically the misclassification rate. Other methods, such as logistic regression, provide an estimate of the exact conditional probability of belonging to the ''special" class given a set of predictor variables. Throughout this paper we shall refer to this conditional probability as ''\emph{conditional risk}" or simply ''\emph{risk}", although sometimes belonging to the special class might actually have a very positive context (e.g. success).

\par There are two ways to estimate the conditional risk function: parametric and non-parametric. Parametric methods primarily include logit/probit models (Martin 1977 \cite{martin_1977}, Ohlsen 1980 \cite{ohlson_1980}) and linear models (Amemiya 1981 \cite{amemiya_1981}, Maddala 1986 \cite{maddala_1986} and Amemiya 1985 \cite{amemiya_1985}). Powell (1994 \cite{powell_1994}) has a review of non-parametric estimators. For a comparison of these approaches and complete review see Elliott and Lieli (2006) \cite{elliott_2006} and more recently Green and Hensher (2010) \cite{greene_2010}.

\par The estimation of the exact structure of the conditional risk function comes in handy when we wish to make distinctions between observations that are finer than simply class membership. However, in realistic scenarios acting upon such estimations alone may prove to be difficult. Assessments on a scale of 1:100 (as percentages) or finer assessments have little practical use, primarily since the possible actions resulting from such information are usually few. For such cases an ordinal output is required. It is important to note that this problem is not equivalent to multi-group classification in two ways: first, our groups are ordinal by nature and relate to the underlying risk; second, the assignment into groups is not given a-priori and greatly depends on the selection of model, model parameters and the borders of the intervals assigned to each risk group.

\par There are two common approaches to creating an ordered set of risk groups to match a finite set of escalating actions. The first approach is to create multiple models describing the behaviour of the conditional risk function at specific points (also known as "quantile regression"); the second approach is to divide post-hoc the continuous risk estimation of a known model into intervals.

\par The first approach attempts to construct separate models that describe the behaviour of the conditional risk function at specific levels of risk. In linear models this approach is known as \emph{quantile regression} (Koenker \& Bassett 1978 \cite{Quantile_regression_1978}). Manski (\cite{Manski_1975}, \cite{Manski_1985}, \cite{Manski_1986}) implemented this notion to binary response models (the equivalent of two-group classification) naming it "Maximum Score Estimation". In a series of papers he shows the existence, uniqueness and optimal properties of the estimators and follows by showing their stable asymptotic properties. The primary justification for using this approach is methodological: it demands that we specify in advance the levels of risk that are of interest to us (a vector $q$ of quantiles), and then constructs a series of models that describe conditional risk at these quantiles. However, as we shall demonstrate in section \ref{section_cond_precentiles}, using risk-quantiles (or conditional probability over left-unbounded and overlapping intervals) is not relevant to our definition of the problem and even the term "conditional quantiles" is in itself misleading.

\par In the second, more ``practical" approach, the continuous output of an existing optimal risk model (logit, linear or non-parametric) is divided into intervals, thus translating the prediction of risk (usually continuous in nature) into large "bins of risk" - i.e "low"/ "medium"/ "high" or "mild"/ "moderate"/ "severe" (depending on context). The final result of this discretization process is a set of ordinal risk groups based on the continuous prediction of conditional risk. The primary drawback of this approach is that the selection of the classification model and its parameters is not performed in light of the final set of desired risk groups. Instead, an "optimal model" (in some sense) is constructed first, and the partition into discrete groups is performed post-hoc.

\par The primary objective of this paper is to combine the idea of pre-set levels of risk over adjacent intervals (rather than risk quantiles) into a standard classification framework. Instead of constructing multiple models, we offer a process that optimizes a single risk estimation model (or "score") paired with a matching set of breakpoints that partition the model's output into ordinal risk groups. To that end we define a new measure of accuracy - \emph{Interval Risk Deviation} (IRD) - which describes a model's ability to distribute risk correctly into intervals given a pre-set vector $r$ of risk levels. We show how this new measure of error can be integrated into existing classification frameworks (specifically the maximum likelihood framework) by adding a constraint to the existing optimization problem. In addition, we address the more practical problem of effectively selecting breakpoints by introducing a penalty function to the modified optimization scheme.

\par The remainder of this paper is organized as follows. Section \ref{section_definitions} defines risk groups and a measure of error (IRD) that will be necessary for optimality. Section \ref{section_existing_methods} demonstrates the problems of using existing approaches.
Section \ref{section_ORGC} formulates a new optimization problem that will provide accurate, optimal and non-degenerate solutions, and section \ref{section_case_study} provides a case study where the new framework is applied to logistic regression and presents an example.

\section{Definitions \label{section_definitions}}

\par Let $r \in [0,1]^T$ be an ordered vector of \emph{risk levels} ($0 \leq r_{1} < r_{2} < \ldots < r_{T} \leq 1$), let $X = (X_1, \ldots, X_P)$ be a continuous $P$-dimensional random vector and let $Y \in \{0,1\}$ be a Bernoulli random variable representing class membership.  An \emph{Ordinal Risk-Group Score} (ORGS) for a pre-set risk vector $r$ is a couplet $(\Psi, \tau)$ where $\Psi: \mathbb{R}^P \rightarrow \mathbb{R}$ is a continuous (possibly not normalized) risk predictor, which summarizes the attributes of $X$ into a single number (a score), and $\tau \in \mathbb{R}^{T-1}$ is a complete partition of $\mathbb{R}$ into $T$ distinct and adjacent intervals ($- \infty  = \tau_0 < \tau_1 < \tau_2 < \ldots < \tau_{T-1} < \tau_T = \infty$). The couplet $(\Psi, \tau)$ classifies observations into risk groups by the following equivalence: An observed vector $X$ belongs to the $i$'th risk group if and only if $\Psi(X) \in (\tau_{i-1}, \tau_{i}]$ (the intervals are right-side open to avoid ambiguities). The actual conditional risk level of the $i$'th risk group defined by a couplet $(\Psi, \tau)$ is:
\begin{equation} \label{eq_R_def}
R_{i} (\Psi, \tau) = P ( Y=1 \mid \Psi(X) \in (\tau_{i-1}, \tau_{i}])
\end{equation}

\par It is worth noting that score-based classification methods for two classes can be described as a special of $T=2$ (two risk groups). Such methods look for a single breakpoint $\tau \in \mathbb{R}$, and the two resulting intervals $(-\infty, \tau], (\tau, \infty)$ become an absolute prediction of class membership: $\Psi(X)> \tau \Rightarrow$ $X$ belongs to class $1$. Other methods, designed to deal with more than one risk group, typically assign a single breakpoint to each risk group (see section \ref{section_cond_precentiles}), reflecting the idea that the assignment to risk group is based on \emph{thresholds}: an observed $X$ is assigned to the $i$'th group if and only if $\Psi(X)$ crosses the ($i-1$)'th threshold ($\Psi(X) > \tau_{i-1}$) but does not cross the $i$'th threshold ($\Psi(X) \leq \tau_{i}$). 

\par Even from the latter definition, it becomes evident that the assignment to groups is in fact based on \emph{adjacent intervals} $\lbrace (\tau_{i-1}, \tau_{i}] \rbrace_{i}^{T-1}$ (rather than on right-side open ended intervals defined by thresholds) ans that any breakpoint we set affects the definition of two intervals (and hence two risk groups). Although further on in this paper we shall discuss separate breakpoints in relation to risk groups in order to demonstrate the key problem that arises from the use of adjacent intervals (section \ref{section_lower_bounds_IRD}), the notion of assigning intervals \emph{simultaneously} rather than separate breakpoints should remain clear throughout this paper.

\par We can now describe the accuracy of an ordinal risk score $(\Psi, \tau)$ in relation to a pre-set vector $r$ as the overall difference between the pre-defined risk levels of $r$ and the actual conditional risk levels $R(\Psi,\tau)$. We define an error measure for risk-group classification models which is a parallel of \emph{misclassification rate} in standard classification methods. We name this measure \emph{Interval Risk Deviation} (IRD):
\begin{equation} \label{eq_IRD_def}
\text{IRD}_{r}(\Psi, \tau) = \Vert R(\Psi, \tau) - r \Vert
\end{equation}

\par On it's own, the very definition of IRD marks a new approach to the evaluation of ordinal risk scores. Having a predefined set of risk levels means that any risk score $(\Psi, \tau)$ we consider as a candidate must uphold $\text{IRD}_{r}(\Psi, \tau) = 0$ (or at the very least $\text{IRD}_{r}(\Psi, \tau) < \varepsilon$ for a predefined small $\varepsilon >0$). This makes $\text{IRD} = 0$ a \emph{necessary condition} for optimality. In the next two sections we demonstrate how the two existing approaches for creating ordinal risk scores do not necessarily fulfil this condition, either because of unsuitable definitions of optimality, as is the case with risk-quantile based methods, or by ignoring it altogether, as is the case with the 2-step approach.

\section{Problems with Existing Scoring Methods \label{section_existing_methods}}

\subsection{Risk-Quantiles (and why we can't use them) \label{section_cond_precentiles}}

\par When first presented with the problem of selecting an optimal model paired with a set of optimal breakpoints, our initial idea was to use quantile-oriented models. Such models have been extensively studied in econometrics, where they are commonly referred to as ``ordered choice models" (\cite{train_2003}, \cite{greene_2010}). The most relevant model in that group is Manski's \emph{maximum score estimation} which defines the optimization problem using a set of probabilities over \emph{left-unbounded overlapping} intervals (or \emph{rays}) in contrast to the definition of the problem over \emph{adjacent, non-overlapping} intervals.

\par In order to better illustrate the differences between our definitions and Manski's quantile-oriented approach we must first describe quantile oriented models in our terms. First we replace the vector $r$ with a vector $q$ of "conditional quantiles", which are in fact the desired conditional probabilities over left-unbounded and overlapping intervals. Using Manski's adaptation of quantile regression \cite{Manski_1975} we can build a different set of model parameters for each quantile $q_{i}$ optimizing:
\begin{equation}
\vert P(Y=1 \mid \Psi_{i}(X) \leq 0) - q_{i} \vert \longrightarrow \min_{\Psi_{i}}
\end{equation}

\par It is easy to see how this approach can be slightly modified to match the original objective of finding a single model: by coercing the models $\Psi_{i}$ to be parallel we can create a "master model" $\Psi(X)$ and derive appropriate thresholds $\lbrace \tau_{i} \rbrace_{i=1}^{T-1}$ such that:
\begin{equation*}
\Psi_{i}(X) \leq 0 \quad \Leftrightarrow \quad \Psi(X) \leq \tau_{i}
\end{equation*}
\begin{equation} \label{optim_cond_quant}
\vert P(Y=1 \mid \Psi(X) \leq \tau_{i}) - q_{i} \vert \longrightarrow \min_{\Psi} \quad i \in \lbrace 1, \dots T \rbrace
\end{equation}

\par Using (\ref{optim_cond_quant}) we can easily define $Q_{i} (\Psi,\tau)$ = $P(Y = 1 \mid \Psi(X) \leq \tau_{i})$ and the equivalent \emph{Quantile Risk Deviation} $QRD_{q}(\Psi, \tau) = \Vert Q(\Psi,\tau) -q \Vert$, and look for a model with $QRD = 0$. However, while it is tempting to describe the vector $q$ as a vector of  ``\emph{conditional quantiles}", the term is in itself misleading and should be avoided. Figure \ref{figure_hetero_cond_graphs} demonstrates how even under relatively simple assumptions (a one dimensional Gaussian distribution with unequal conditional variances) the function $Q_{i} (\Psi,\tau) = P(Y = 1 \mid X \leq \tau_{i})$ is not even monotone in $\tau_{i}$.
\begin{figure}[h]
\center
\includegraphics[scale=0.9, angle = 270]{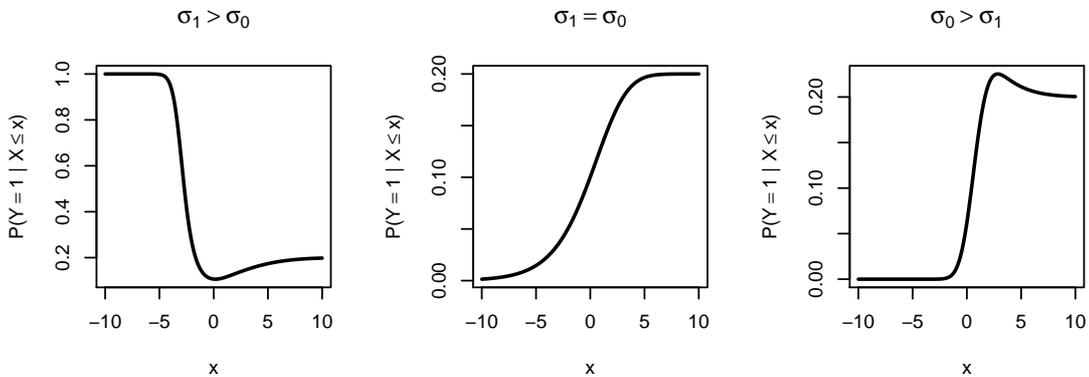}
\caption[Conditional probability over left-unbounded intervals]{\label{figure_hetero_cond_graphs} Different behaviour of conditional probability over left-unbounded intervals as a function of the threshold $x$ in the case of one-dimensional Gaussian distribution with $\mu_{0}= -1$, $\mu_{1} = 1$ and $P(Y=1) = 0.2$. In the left panel $\sigma_1 =4, \sigma_0 = 1$, in the middle panel $\sigma_1 = \sigma_0 = 2$ (homoscedastic case) and in the right panel $\sigma_1 = 1, \sigma_0 = 4$.}
\end{figure}
\par Even if we assume strict monotonicity of $P(Y=1 \mid \Psi(X) \leq x)$, for example by assuming the strict monotone likelihood ratio property (SMLRP, for details see Appendix \ref{appendix_MLRP_cdf_ratio_mono}) and thus giving the term ``conditional quantiles" a meaningful sense, it would still be impossible to apply this approach to optimizing the distribution of risk over adjacent intervals. In order to use ``risk-quantiles" to solve our problem we must first find an a-priori mechanism that will translate any given vector of desired conditional probabilities over adjacent intervals $r$ to the equivalent vector of desired conditional probabilities over left unbounded and overlapping intervals $q$. 

\par However it is easy to show that such an a-priori translation is impossible. Using the \emph{law of total probability} in its conditional form we can calculate for any given $R$ the equivalent $Q^{(R)}$ (actual probabilities over left unbounded intervals):
\begin{equation} \begin{split} \label{eq_R_Q}
Q^{(R)}_{i} & (\Psi,\tau) = P(Y = 1 \mid \Psi(X) \leq \tau_{i})\\
= & \sum_{j \leq i} P(Y=1 \mid \Psi(x) \in (\tau_{j-1}, \tau_{j}], \Psi(X) \leq \tau_{i}) P(\Psi(X) \in (\tau_{j-1}, \tau_{j}] \mid \Psi(X) \leq \tau_{i})\\
= & \sum_{j \leq i} P(Y=1 \mid \Psi(x) \in (\tau_{j-1}, \tau_{j}]) \frac{P(\Psi(X) \in (\tau_{j-1}, \tau_{j}] ,\Psi(X) \leq \tau_{i})}{P( \Psi(X) \leq \tau_{i})} \\
= & \frac{1}{P(\Psi(X) \leq \tau_{i})} \sum_{j \leq i} R_{j}(\Psi,\tau) P(\Psi(X) \in (\tau_{j-1}, \tau_{j}])
\end{split} \end{equation}
Or equivalently:
\begin{equation}
R_{i} (\Psi,\tau) = \frac{P(\Psi(X) \leq \tau_{i})}{P(\Psi(X) \in (\tau_{i-1}, \tau_{i}])} Q^{(R)}_{i} (\Psi,\tau) - \frac{P(\Psi(X) \leq \tau_{i-1})}{P(\Psi(X) \in (\tau_{i-1}, \tau_{i}])} Q^{(R)}_{i-1}(\Psi,\tau)
\end{equation}
The same process can be applied to the corresponding vector of risk quantiles $q^{(r)}$:
\begin{equation}  \begin{split} \label{eq_r_q}
q^{(r)}_{i}  = & \frac{\sum_{j < i} \; r_{j} \; P(\Psi(X) \in (\tau_{i-1}, \tau_{i}])}{P(\Psi(X) \leq \tau_{i})} \\
 r_{i} = & \frac{P(\Psi(X) \leq \tau_{i})}{P(\Psi(X) \in (\tau_{i-1}, \tau_{i}])} q^{(r)}_{i} - \frac{P(\Psi(X) \leq \tau_{i-1})}{P(\Psi(X) \in (\tau_{i-1}, \tau_{i}])} q^{(r)}_{i-1}
\end{split} \end{equation}
As a result for a fixed $(\Psi, \tau)$ we have:
\begin{equation*}
R_{i} (\Psi,\tau) = r_i \Leftrightarrow Q^{(R)}(\Psi,\tau) = q^{(r)}_{i}
\end{equation*}
\begin{equation}
\text{IRD}_{r}(\Psi,\tau) = 0 \Leftrightarrow QRD_{q^{(r)}}(\Psi,\tau) = 0
\end{equation}

\par The primary problem of using quantiles to define this problem stems from the relation between $r$ and the resulting $q_{r}$. By our own definitions the central aspect of the problem is the probability over adjacent intervals and not overlapping left-unbounded intervals. Therefore the optimization must be performed against a fixed, pre-defined vector $r$. If we wish to construct an analogous quantile-based optimization problem, we must first find the equivalent vector $q_r$ which defines quantile-based problem. However equation (\ref{eq_r_q}) shows that since the relation between $r$ and $q_r$ depends on the specific form of the optimal model $\Psi$, in order to construct $q_r$ we must first find the optimal model $\Psi$ for this problem (which is what we are looking for in the first place), or in other words the translation $r \leftrightarrow q$ is possible only once we have the optimal solution to the problem. Therefore building an analogous optimization problem over left-unbounded overlapping intervals can only be done \emph{after} we have the optimal solution. Consequently we cannot use quantile-based models to construct an optimal model for the adjacent interval-based ordinal risk-group problem.

\subsection{Lower bounds on Interval Risk Deviation \label{section_lower_bounds_IRD}}

\par Another common practice when building scores for risk groups is to build a model $\Psi$ that is optimal in some sense (e.g. maximizing likelihood or minimizing overall miss-classification rate) and then partition the range of $\Psi(X)$ into adjacent intervals the define risk groups. In this section we demonstrate how, under relatively simple assumptions, using this approach with existing classification models is not optimal for more than two risk groups. 

\par Using Bayes theorem we can represent $R$ as:
\begin{equation*}
R_{i} (\Psi ,\tau) = P(Y=1 \mid \Psi(X) \in (\tau_{i-1}, \tau_{i}]) = P(Y=1)\frac{P(\Psi(X) \in (\tau_{i-1}, \tau_{i}] \mid Y=1)}{P(\Psi(X) \in (\tau_{i-1}, \tau_{i}])}
\end{equation*}

\par We assume that $(X,Y,\Psi)$ satisfies the Strict Monotone Likelihood Ratio Property (SMLRP, see appendix \ref{appendix_MLRP_cdf_ratio_mono} for exact definition and details) and that the marginal densities $f_{X \mid Y=k}$ ($k = 0,1$) are continuous, strictly positive and finite. By continuity and finiteness we can describe the behaviour of $R_{i} (\Psi,\tau)$ for infinitely short intervals ($\tau_{i} \rightarrow \tau_{i-1}$):
\begin{equation} \label{eq_lim_of_R} \begin{split}
& \lim_{\tau_{i} \rightarrow \tau_{i-1}} R_{i} (\Psi,\tau) =  \lim_{\tau_{i} \rightarrow \tau_{i-1}} \frac{P(Y=1) \; P(\Psi(X) \in (\tau_{i-1}, \tau_{i}] \mid Y = 1)}{P(\Psi(X) \in (\tau_{i-1}, \tau_{i}])} = \\
& = P(Y=1)  \frac{\lim_{\tau_{i} \rightarrow \tau_{i-1}}  \frac{P(\Psi(X) \in (\tau_{i},\tau_{i-1}] \mid Y = 1)}{\tau_{i} - \tau_{i-1}}}{\lim_{\tau_{i} \rightarrow \tau_{i-1}}\frac{P(\Psi(X) \in (\tau_{i},\tau_{i-1}])}{\tau_{i} - \tau_{i-1}}} = P(Y=1) \frac{f_{\Psi(X) | Y=1} (\tau_{i-1})}{f_{\Psi(X)} (\tau_{i-1})}
\end{split} \end{equation}
where $f$ is the appropriate density function and the limit is from the right-hand side. Similarly for any $z \in (\tau_{i-1}, \tau_{i}]$,
\begin{equation}  \label{eq_lim_of_R_general}
\lim_{\tau_{i} \rightarrow z} \lim_{\tau_{i-1} \rightarrow z} R_{i} (\Psi,\tau) = \lim_{\tau_{i-1} \rightarrow z} \lim_{\tau_{i} \rightarrow z} R_{i} (\Psi,\tau) = P(Y=1) \frac{f_{\Psi(X) | Y=1} (z)}{f_{\Psi(X)} (z)}
\end{equation}

\par Although we have stressed the importance of simultaneity when assigning intervals to risk groups, in order to understand the implications of (\ref{eq_lim_of_R_general}) on optimal model selection we must look at the problem from a different perspective. First we fix $\Psi$ and assume that a given  partition $\tau$ supports a perfect distribution of conditional risk up to the $(i-1)$'th group, meaning that $R_{j}(\Psi, \tau) = r_{j}$ for all $j < i$. Under these conditions, combined with our previous assumptions of continuous, strictly positive conditional densities and SMLRP, we can explicitly show that not all values of $r_{i}$ are exactly achievable without introducing some IRD: by theorem \ref{th_SMLRP_mono_R_equiv} $R_{i}(\Psi, \tau)$ is strictly increasing in $\tau_{i}$ and therefore we can explicitly define a feasibility criterion:
\begin{equation} \label{eq_lower_bound_on_r}
P(Y=1) \frac{f_{\Psi(X) | Y=1} (\tau_{i-1})}{f_{\Psi(X)} (\tau_{i-1})} < r_{i}
\end{equation}
\par Using continuity (which enables us to divide by $P(Y=1)f_{\Psi(X) | Y=1} (\tau_{i-1})$) we can transform (\ref{eq_lower_bound_on_r}) into a condition on the likelihood ratio $\Lambda$:
\begin{equation} \label{eq_lower_bound_on_LR}
\Lambda_{\Psi}(\tau_{i-1}) = \frac{f_{\Psi(X) \mid Y=1} (\tau_{i-1})}{f_{\Psi(X) \mid Y=0} (\tau_{i-1})} < \frac{1 - P(Y=1)}{P(Y=1)} \; \frac{r_{i}}{1-r_{i}}
\end{equation}
If $\tau$ does not meet the feasibility criterion (\ref{eq_lower_bound_on_r}), then by (\ref{eq_lim_of_R}) and strict monotonicity of $R$ any selection of $\tau_{i} > \tau_{i-1}$ will have $R_{i} (\Psi,\tau) > r_{i}$ even if we set the interval $(\tau_{i-1}, \tau_{i}]$ to be arbitrarily small. The inevitable result that, for the our fixed model $\Psi$, \emph{any} choice of $\tau$ will have $\text{IRD}_{r}(\Psi, \tau) > 0$.

\par It is important to note that the set of $T-1$ inequalities defined by (\ref{eq_lower_bound_on_r}), (\ref{eq_lower_bound_on_LR}) are necessary yet not sufficient conditions for IRD=0. Assume that we have a solution $(\Psi, \tau)$ which satisfies $\text{IRD}_{r}(\Psi, \tau) = 0$. Under SMLRP we have $x_{2} > x_{1} \: \Rightarrow \: \Lambda_{\Psi}(x_{2}) > \Lambda_{\Psi}(x_{1})$. Our counter example $(\Psi, \tilde{\tau})$ satisfies $\tilde{\tau_{1}} < \tau_{1}$ and $\forall i > 1 : \: \tilde{\tau_{i}} = \tau_{i}$ . By SMLRP we have:
\begin{equation*} \begin{split}
& P(Y=1) \frac{f_{\Psi(X) | Y=1} (\tilde{\tau_{1}})}{f_{\Psi(X)} (\tilde{\tau_{1}})} = \left(1 + \frac{1-p}{p} \frac{1}{\Lambda_{\Psi}(\tilde{\tau_{1}})} \right)^{-1} < \\
< & \left(1 + \frac{1-p}{p} \frac{1}{\Lambda_{\Psi}(\tau_{1})} \right)^{-1} = P(Y=1) \frac{f_{\Psi(X) | Y=1} (\tau_{1})}{f_{\Psi(X)} (\tau_{1})} < r_{2}
\end{split} \end{equation*}
Therefore (\ref{eq_lower_bound_on_r}) is maintained (the other inequalities are not affected). On the other hand by theorem \ref{th_SMLRP_mono_R_equiv} we have strict monotonicity of $R$, meaning: 
\begin{equation*}
R_{1}(\Psi,\tilde{\tau}) = P(Y=1 \mid \Psi(X) < \tilde{\tau_{1}}) < P(Y=1 \mid \Psi(X) < \tau_{1}) = R_{1} (\Psi,\tau) = r_{1}
\end{equation*}
and therefore $\text{IRD}_{r}(\Psi, \tilde{\tau}) > 0$. The conclusion is that even under SMLRP we can use (\ref{eq_lower_bound_on_r}),(\ref{eq_lower_bound_on_LR}) only as necessary conditions for the feasibility of a given solution and that the test of feasibility must be performed using (\ref{eq_R_def}) and (\ref{eq_IRD_def}) directly. 

\par In order to satisfy the necessary conditions for IRD=0 in the absence of SMLRP we can generally require $r_{i} > \underset{\lbrace \tau_{i}: \tau_{i} > \tau_{i-1} \rbrace}{\inf} R_{i} (\Psi,\tau)$ (we require strong inequalities to avoid degenerate zero-length intervals), however for such cases the existence of a closed-form expression would depend on the exact distribution of $X | Y=k$ ($k = 0,1$). We leave the exact formulation of non-SMLRP lower bounds outside the scope of this paper.

\par The final conclusion is that given two sets of risk categories $r_{1}, r_{2}$ and a couplet $(\Psi, \tau_{1})$ which satisfies $\text{IRD}_{r_{1}}(\Psi,\tau_{1}) = 0$, we may not be able to find a set of breakpoints $\tau_{2}$ which satisfies $\text{IRD}_{r_{2}}(\Psi,\tau_{2}) = 0$ (using the same model $\Psi$). Specifically we can now claim that optimal models of existing classification methods (typically optimized for $r=(0,1)$) are not necessarily feasible for any choice of $r$. 

\par The existence of lower bounds on the IRD is perhaps the most counter-intuitive result of this paper. The reason why these limitations have not been addressed before has to do with the fact that most classification methods use a single breakpoint to distinguish between the two groups ($\tau \in \mathbb{R}$) and the issue of degenerate solutions or non-feasibility of $\Psi$ is avoided altogether. Although the fulfilment of (\ref{eq_lower_bound_on_r}),(\ref{eq_lower_bound_on_LR}) does not ensure the feasibility of a given solution, these inequalities are instrumental in demonstrating why the solutions from existing methods may not be feasible for a different choice of $r$, and provide an elegant method to disqualify such solutions. Once we define our objective as the distribution pre-set risk levels over multiple adjacent intervals we must recognize the existence of possible limitations on IRD for existing methods and as a result define new conditions for optimality.

\section{Ordinal Risk-Group Classification \label{section_ORGC}}

\par Although the definition of IRD naturally suggests itself as a new criterion for optimality (look for a couplet $(\Psi, \tau)$ such that $\text{IRD}_{r}(\Psi, \tau(\Psi)) = 0$), there are two problems with using IRD as a single optimality criterion. First, since our problem is a classification problem we must consider some sense of the quality of separation between the two classes in order to avoid degenerate solutions. This principle is not straight forwardly reflected by the definition of IRD (\ref{eq_IRD_def}). Second, our definition of IRD and the resulting necessary inequalities (\ref{eq_lower_bound_on_r}) do not ensure existence or uniqueness of an optimal solution. 

\par Our practical solution to these problems is to define IRD as a feasibility criterion and use it as a constraint in an existing optimization problem. Since we are still in the domain of classification problems it would be reasonable to preserve some basic concepts, particularly the definition of optimality: We seek a model that on the one hand maximizes our ability to discriminate between the two classes, but on the other hand distributes risk correctly, meaning that it belongs to the set of feasible solutions:
\begin{equation}
C_{r}(0) = \lbrace (\Psi,\tau) : \: \text{IRD}_{r}(\Psi, \tau) = 0 \rbrace
\end{equation}
In the event that $C$ is an empty set we would have to reconsider our pre-set $r$ or change our method of constructing $\Psi$. 

\par Any classification method we might consider for IRD ``augmentation" must satisfy several criteria. First, it must provide a continuous output $\Psi(X)$, ensuring that we have an appropriate output that can be partitioned into intervals (using $\tau$). This requirement automatically excludes classification methods that do not combine the vector of explanatory variables $X$ into a single real-valued score $\Psi(X)$ before making a prediction of risk or class membership (classification trees are an example of such excluded methods). Furthermore, we would like to maintain the notion that observations with higher scores have a higher conditional risk, and therefore require that the output $\Psi(X)$ is strongly correlated with the conditional risk function $P(Y=1 \mid \Psi(X)=x)$. Methods such as Fisher's LDA \cite{fisher_LDA} or SVM for two classes do provide a continuous scale and a single breakpoint to predict class membership, however these scales are not necessarily correlated with the conditional risk and only ensure that a majority of the observations from the special class are on one side of the breakpoint. We therefore decided to focus our discussion on risk estimation methods that provide a direct estimation of the risk function:
\begin{equation}
\Psi: \mathbb{R}^{P} \longrightarrow [0,1] \:, \quad 
\Psi(X) = P(Y=1\mid X)
\end{equation}

\par Finally, in order to simplify our construction we assume SMLRP (see appendix \ref{appendix_MLRP_cdf_ratio_mono}). As we have seen before, this assumption ensures that we have a simple way to calculate the lower bounds on IRD, and also ensures that for a given model $\Psi$, if exists $\tau(\Psi)$ such that $\text{IRD}_{r}(\Psi, \tau(\Psi)) = 0$ then it is unique (see lemma \ref{lemma_tau_unique}). These properties enable us to simplify our parameter space by optimizing over $\Psi$ alone, and provide a simple way to test for the existence of necessary conditions for $\text{IRD}_{r}(\Psi) = \text{IRD}_{r}(\Psi, \tau(\Psi))=0$ and optimize under the constraint $C_{r}(\Psi) = \lbrace \Psi : \: \text{IRD}_{r}(\Psi) = 0 \rbrace $.

\subsection{Penalized Optimization \label{section_penalty}}

\par While the idea of fitting an optimal risk predictor that maximizes class discrimination is a well defined concept, the requirement of $\text{IRD}_{r}(\Psi)=0$ may lead to degenerate solutions of $\tau(\Psi)$ for certain values of $r$. We demonstrate this problem for a simple case of homoscedastic one-dimensional Gaussian logistic regression: Let $X \mid Y = 1 \sim N(\mu, \sigma)$, $X \mid Y = 0 \sim N(-\mu, \sigma)$, $P(Y=1) = \frac{1}{2}$ and the model $\Psi(\beta,x)$ is the one-dimensional logistic function with the parameter $\beta$, meaning $\Psi: \mathbb{R} \times \mathbb{R} \longrightarrow [0,1]$, $\Psi(\beta, x) = \frac{\exp(\beta x)}{1+\exp(\beta x)}$. We set $r = (0,0.5,1)$.

\par Denoting $\tau(\Psi, \beta, t) = (\Psi(\beta,-t),\Psi(\beta,t))$, we use symmetry of the conditional distributions around $x=0$ and the strict monotonicity of $\Psi(\beta, x)$ in $x$ and $\beta$ to show that for any choice of $\beta,t \in \mathbb{R}$ we can minimize error for $i=2$:
\begin{equation*} \begin{split}
 R_{2} & (\Psi(\beta,X), \tau(\Psi, \beta, t)) = P(Y = 1 \mid  \Psi(\beta, X) \in (\Psi(\beta,-t),\Psi(\beta,t)]) \\
& = P(Y = 1 \mid  \beta X \in (- \beta t, \beta t]) =  P(Y = 1 \mid  X \in (-t, t]) = 0.5 = r_{2}
\end{split} \end{equation*}
Similar considerations ensure that the risk prediction errors are equal on both sides:
\begin{equation*} \begin{split}
 R_{1} & (\Psi(\beta,X), \tau(\Psi, \beta, t)) = P(Y=1 \mid X<-t) \\
& = 1 - P(Y=1 \mid X>-t) =  1 - R_{3} (\Psi(\beta,X), \tau(\Psi, \beta, t))
\end{split} \end{equation*}
Using Bayes theorem we have:
\begin{equation*} \begin{split}
P(Y=1 \mid X<-t) = & \frac{P(Y=1)P(X < -t \mid Y=1)}{P(Y=1)P(X < -t \mid Y=1) + P(Y=0)P(X < -t \mid Y=0)} \\
 = & \frac{P(Y=1) \Phi(-t - \mu)}{P(Y=1) \Phi(-t - \mu) + P(Y=0) \Phi(-t + \mu)} \\
 = & \left( 1 + \frac{P(Y=0)}{P(Y=1)} \; \frac{\Phi(-t + \mu)}{\Phi(-t - \mu)}  \right)^{-1}
\end{split} \end{equation*}
where $\Phi$ is the CDF of the standard normal distribution. Using the known inequality:
\begin{equation} \label{eq_Phi_approx}
\frac{\phi(x)}{x + 1/x} < \Phi(-x) < \frac{\phi(x)}{x}  \quad
\forall x>0,
\end{equation}
where $\phi$ is the PDF of the standard normal distribution, we show
an upper bound:
\begin{equation*}
\frac{\Phi(-t + \mu)}{\Phi(-t - \mu)} > \frac{(t+\mu) + \frac{1}{t+\mu}}{t-\mu} \frac{\phi(t-\mu)}{\phi(t+\mu)} = \frac{(t+\mu) + \frac{1}{t+\mu}}{t-\mu} \: e ^{2 \mu t} \quad \forall t>\mu
\end{equation*}
Therefore $\lim_{t \rightarrow \infty} P(Y=1 \mid X<-t) = 0$ and similarly $\lim_{t \rightarrow \infty} P(Y=1 \mid X>t) = 1$. For any arbitrarily small $\varepsilon >0$ we can find a sufficiently large $t$ such that
\begin{equation*}
R_{1} (\Psi(\beta,X), \tau(\Psi, \beta, t)) = P(Y=1 \mid X<-t) =  \leq \varepsilon / 2
\end{equation*}
making the total IRD:
\begin{equation*}
\text{IRD}_{r} (\Psi(\beta,X),  \tau(\Psi, \beta, t))  = \sqrt{\sum_{i=1}^3 \left( R_{i} (\Psi(\beta,X), (\Psi(\beta,-t),\Psi(\beta,t))) - r_{i} \right)^2} \leq \varepsilon
\end{equation*}
As a result, for any given $\beta$ the only solution that satisfies IRD=0 is degenerate:
\begin{equation*}
\lim_{t \rightarrow \infty} \text{IRD}_{r} (\Psi(\beta,X),  \tau(\Psi, \beta, t)) = 0
\end{equation*}

\par There are several alternatives for dealing with this problem. First, we may decide that methodologically we do not allow setting $r_{1} = 0$ or $r_{T} = 1$. This will ensure that the values of $\tau$ are finite but might still lead to very large or very small intervals, depending on the parameters of the model. Alternatively, if our risk estimation method uses optimization to fit the optimal model (for example maximizing the likelihood function in the case of parametric methods) then we can introduce a penalty function $\text{Pen}: \mathbb{R}^{T-1} \rightarrow \mathbb{R}$, which will enable us to balance the properties of $\tau$ (minimal or maximal distance between breakpoints) with the discrimination properties of $\Psi$. This means that instead of maximizing or minimizing a target function $f(\Psi \mid X,Y)$ we maximize/minimize $f(\Psi \mid X,Y) + \gamma \text{Pen}(\tau)$ under an IRD constraint, where $\gamma$ is a tuning parameter that represents the degree of aversion to degenerate solutions.

%
%
%

\par In cases where the degenerate solutions are encountered we would opt for the use of a penalty function. This reflects our understanding that the requirement of ``evenly spread" breakpoints is relatively subjective and should allow for some discretion as to the balance between the ability of the model to separate classes and the resulting interval lengths. By choosing an appropriate penalty function and an aversion parameter $\gamma$ we enable  better fitting of the model according to the circumstances at hand, while introducing a relatively small number of additional parameters. On the other hand, since IRD represents an absolute measure of the model's quality, we believe it must be tightly controlled as the constraint $\text{IRD}_{r} (\Psi, \tau) = 0 $ on any model we might consider. We address the details of constructing this constraint for parametric models in the following section.

\subsection{Estimation of Interval Risk Deviation \label{sec_IRD_est}}

\par So far, we have defined interval risk deviation (IRD) as a property of a score model $\Psi$ and the joint distribution of $(X,Y)$. In order to implement the concept of IRD in a real-life scenario we must describe a way to estimate $\text{IRD}_{r}(\Psi)$ based on a sample of $N$ i.i.d observations from a known $P$-dimensional multivariate distribution $\mathcal{F}(\theta)$ in the form of a $N \times P$ matrix $\mathbf{X}$ and a vector $y \in \lbrace 0,1 \rbrace^{N}$ representing known class memberships (depending on the design of the experiment $y$ may or may not be a random sample). Focusing on parametric methods, we assume that $ P(Y=1 \mid X) = \Psi(\beta,X)$ where $\Psi: \mathbb{R}^{M} \times \mathbb{R}^{P} \rightarrow [0,1]$ is a known function and $\beta  \in \mathbb{R}^{M}$ is the set of parameters controlling the shape of the function (e.g. generalized linear models \cite{GLM_1972} where $M = P$, $g$ is a known, strictly monotone and bijective link function and $\Psi(\beta,x) = g(\beta^{T} x$)). Having previously assumed SMLRP and a closed-from $\Psi$, we can simplify our notation by denoting $\tau(\beta) = \tau(\Psi(\beta,X))$, IRD for a given $\beta$ as $\text{IRD}_{r} (\beta) = \text{IRD}_{r} (\Psi(\beta,X), \tau(\Psi(\beta,X)))$ and the constraint set $C_{r}(\beta) = \lbrace \beta \in \mathbb{R}^{M} : \: \text{IRD}_{r}(\beta) = 0 \rbrace$

\par Many parametric classification methods solve the problem of estimating $\beta$ from a given sample $(\mathbf{X},y)$ by using the maximum likelihood (ML) method. We denote $\mathbf{X}_{j, \cdot}$ the $j$'th row of the matrix $\mathbf{X}$, making our model-predicted probability for the $j$'th observation $\Psi(\beta,\mathbf{X}_{j, \cdot}) = P(Y=1 \mid \mathbf{X}_{j, \cdot})$. Assuming random sampling, there are two equivalent formulations of the ``complete" likelihood function:
\begin{equation} \label{eq_likelihood_cond}
L(\theta_{0}, \theta_{1}, p \mid \mathbf{X},y) = \prod_{j=1}^{N} f_{X,Y} (\mathbf{X}_{j, \cdot}, y_{j}) =  \prod_{j=1}^{N} f_{\theta_{y_{j}}} (\mathbf{X}_{j, \cdot}) P(Y_{j} = y_{j})
\end{equation} \begin{equation} \label{eq_likelihood_beta}
L(\beta, \theta \mid \mathbf{X},y) = \prod_{j=1}^{N} P(Y = y_{j} \mid \mathbf{X}_{j, \cdot}) f_{\theta} (\mathbf{X}_{j, \cdot})
\end{equation}
where $f_{\theta}$ is the density function of the distribution $\mathcal{F}(\theta)$ of $X$ and $f_{\theta_{y_{j}}}$ is the density function of the conditional distribution $\mathcal{F}_{y_{j}}(\theta_{y_{j}})$ of $X \mid Y = y_j$. When using (\ref{eq_likelihood_cond}) we must make additional assumptions about the conditional distribution of $X \mid Y=k$ and a random sampling process, and as a result our estimator $\hat{\beta}$ becomes a function of the estimators $\hat{\theta}_{0}, \hat{\theta}_{1}, \hat{p}$. On the other hand using (\ref{eq_likelihood_beta}) can significantly simplify the optimization process. Since our parameter of interest $\beta$ is isolated in the term $P(Y = 1 \mid \mathbf{X}_{j, \cdot}) = \Psi(\beta,\mathbf{X}_{j, \cdot}) $ and does not effect the term $f_{\theta} (\mathbf{X}_{j, \cdot})$, we can directly maximize the the partial likelihood function $L_{\Psi}$ over the values of $\beta$:
\begin{equation} \label{qe_likelihood_Psi}
L_{\Psi} (\beta \mid \mathbf{X},y) = \prod_{j=1}^{N} P(Y = y_{j} \mid \mathbf{X}_{j, \cdot}) = \prod_{j=1}^{N} \Psi(\beta, \mathbf{X}_{j, \cdot})^{y_{j}} (1 - \Psi(\beta, \mathbf{X}_{j, \cdot}))^{1-y_{j}}
\end{equation}
making the corresponding maximum likelihood optimization problem:
\begin{equation} \label{eq_ML_Psi}
\hat{\beta}_{ML} = \underset{ \beta \in \mathbb{R}^{M} }{ \text{argmax} } \; L_{\Psi} (\beta \mid \mathbf{X},y) 
\end{equation} 
One of the primary advantages of using the second approach for maximum likelihood estimation is that it circumvents the need to estimate the parameters of the conditional distributions, thus making $\beta$ the only estimated parameter. This construction also enables a relatively simple extension of the maximum likelihood framework to semi-parametric models or non-random sampling (see \cite{PRENTICE01121979} for such an extension of logistic regression).

\par Incorporating the IRD constraint into the maximum likelihood framework means that for any given $\beta$ we must be able to estimate the conditional probability over intervals $R_{i}(\beta, \tau) = P(Y = 1 \mid \Psi(\beta, X) \in (\tau_{i-1}, \tau_{i}])$ for an arbitrary $\tau$ from the sample $(\mathbf{X},y)$, which we can then use to calculate the estimate for the unique optimal $\tau(\beta)$ (which is a function of both $\beta$ and the distribution of $X$). Alas, it is usually difficult to derive $R_{i}(\beta,\tau)$ directly from the point-wise conditional probability $P(Y = 1 \mid X)$. In order to facilitate the estimation of IRD for such cases we use an approach similar to (\ref{eq_likelihood_cond}) but in a different context. As we shall see this will enable us to provide parametric estimates of IRD while retaining the convenient structure of our target function $L_{\Psi}$. 

\par We assume that $\Psi(\beta, X) \mid Y=k$ has a known distribution which we denote $F(\eta_{k}({\beta}))$, and that $F$ is continuous. We can now use Bayes's theorem to represent $R(\beta, \tau)$ as:
\begin{equation} \label{eq_nu_R}
R(\beta, \tau) = R(\Psi(\beta,X), \tau) =  \left( 1 + \frac{1-p}{p} \: \frac{1}{\nu(\beta, \tau)} \right)^{-1}
\end{equation}
where:
\begin{equation} \label{eq_nu_def}
\nu_{i}(\beta, \tau) = \frac{P(\Psi(\beta, X) \in (\tau_{i-1},\tau_{i}] \mid Y=1)}{P(\Psi(\beta, X) \in (\tau_{i-1},\tau_{i}] \mid Y=0)} = \frac{F_{\eta_{1}(\beta)}(\tau_{i}) - F_{\eta_{1}(\beta)}(\tau_{i-1})}{F_{\eta_{0}(\beta)}(\tau_{i}) - F_{\eta_{0}(\beta)}(\tau_{i-1})}
\end{equation}
Therefore by estimating the parameters $p$, $\eta_{0}(\beta)$, $\eta_{1}(\beta)$ we can calculate the estimators $\hat{\nu}(\beta, \tau)$ and $\hat{R}(\beta, \tau)$ and use them to find the unique $\hat{\tau}(\beta)$ which solves:
\begin{equation}
\hat{\text{IRD}}_{r}(\beta) = \text{IRD}_{r}(\beta, \hat{\tau}(\beta)) = 0
\end{equation}
The complete likelihood function under these assumptions is:
\begin{equation}
L(\eta_{0}(\beta), \eta_{1}(\beta), p \mid \mathbf{X}, y, \beta) = \prod_{j=1}^{N} f_{\eta_{y_{j}}(\beta)} ( \Psi(\beta, \mathbf{X}_{j, \cdot}) ) P(Y_{j} = y_{j})
\end{equation}
where $f_{\eta_{y_{j}}(\beta)}$ is the density function of $\Psi(X) \mid Y = y_{j} \sim F(\eta_{y_{j}}({\beta}))$. Since the estimation of $p$, $\eta_{0}(\beta)$, $\eta_{1}(\beta)$ is performed for each $\beta$ independently and only in order to verify the compliance with the constraint $\hat{\text{IRD}}_{r}(\beta)=0$, the estimation of these parameters does not effect the value of the target function $L_{\Psi}(\beta \mid \mathbf{X}, y)$. We use this property together with the separability of the likelihood function to estimate $p$ and $\eta_{0}(\beta), \eta_{1}(\beta)$ separately for each beta. For the case of random sampling, the parameter $p$ is relatively easy to estimate independently of $\beta$ as:
\begin{equation}
\hat{p} = \hat{P}(Y=1)=\frac{\sum_{j=1}^{N} y_{i}}{N}
\end{equation} 
although for other cases, such as case-control studies, we might need additional information to correct for biased sampling. For the estimation of $\eta_{0}(\beta), \eta_{1}(\beta)$ we use our assumptions about $F$ to build an ancillary optimization problem for each $\beta$ individually, and use the estimates provided by the ancillary problem (conditional on the value of $\beta$) to estimate $\hat{\text{IRD}}_{r}(\beta, \tau)$ (again the target function $L_{\Psi}$ remain unchanged). The likelihood function for the ancillary problem can be rewritten as:
\begin{equation}
L_{F}(\eta_{0}(\beta), \eta_{1}(\beta) \mid \mathbf{X}, y, \beta) = \prod_{\lbrace j: \: y_{j} = 0 \rbrace} f_{\eta_{0}(\beta)} ( \Psi(\beta, \mathbf{X}_{j, \cdot}) )  \prod_{\lbrace j: \: y_{j} = 1 \rbrace} f_{\eta_{1}(\beta)} ( \Psi(\beta, \mathbf{X}_{j, \cdot}) )
\end{equation}
where $f_{\eta_{k}(\beta)}$ is the density function of the distribution $\Psi(\beta, X) \mid Y=k \sim F(\eta_{k}(\beta))$, and the ancillary maximum likelihood estimation problem is:
\begin{equation}
(\hat{\eta}_{0}(\beta), \hat{\eta}_{1}(\beta))  = \underset{\eta_{0}(\beta), \eta_{1}(\beta)}{\text{argmax}} \: L_{F} (\eta_{0}(\beta), \eta_{1}(\beta) \mid \mathbf{X},y, \beta) 
\end{equation}

\par The construction of multiple ancillary ML estimation problems can be computationally demanding, but luckily for many known distributions the formula for the ML estimator $\hat{\eta}_{k}(\beta)$ is known, and in fact it is relatively simple to calculate it directly from the known ML estimators of the conditional distribution $X \mid Y=k \sim \mathcal{F}(\theta_{k})$ as $\hat{\eta}_{k}(\beta) = \eta(\beta, \hat{\theta}_{k})$. This fact significantly reduces the complexity of estimating the IRD constraint.
For a concrete example of such a case using Gaussian logistic regression (which can be easily extended to other GLM instances) see section \ref{section_case_study}. 

\par Alternatively it would be possible to use non-parametric estimators or approximation. While it is possible to attempt to directly approximate $Q_{i}(\beta) = P(Y = 1 \mid \Psi(\beta, X) \leq \tau_{i}(\beta))$, $P( \Psi(\beta, X) \leq \tau_{i-1}(\beta))$, $P( \Psi(\beta, X) \leq \tau_{i}(\beta))$ and utilize (\ref{eq_R_Q}) for a direct estimation of $R(\beta)$ (in this case the equality is useful since $\beta$ is fixed), it would often be more convenient to approximate $p$ and $F_{\eta_{k}(\beta)}$ at $\lbrace \tau_{i}(\beta) \rbrace_{i=1}^{T-1}$ (a total of $2T-1$ approximations per $\beta$) and use (\ref{eq_nu_R}) to calculate $\hat{\nu}(\beta)$ and the resulting IRD estimate. 
The primary advantage of this approach is that it requires no additional assumptions about the distribution of $(X,Y)$ and can therefore be easily extended to other non-ML estimation methods of $\beta$. On the other hand, by using non-parametric methods we pay a price both in the quality of our estimates (we ignore information about the distribution of $X$) and in the computational complexity of our estimation scheme.

\par Finally, regardless of our approach to estimation, we recognize the fact that under realistic scenarios we will have to use numeric methods for the calculation of $\hat{\tau}(\beta)$ and for the estimation of the required parameters. We therefore set a low threshold $\varepsilon$ and accept $\beta$ as feasible if our estimated IRD satisfies:
\begin{equation} \label{eq_est_IRD_constr}
\hat{\text{IRD}}_{r}(\beta) = \text{IRD}_{r}(\beta, \hat{\tau}(\beta)) = \Vert \hat{R}(\beta, \hat{\tau}(\beta)) - r \Vert < \varepsilon
\end{equation}
making our feasibility set $\hat{C}_{r}(\varepsilon) = \lbrace \beta \in \mathbb{R}^{M} : \: \hat{\text{IRD}}_{r}(\beta) < \varepsilon \rbrace$ and the penalized Ordinal Risk-Group (ORG) optimization problem:
\begin{equation} \label{optim_constr_pen}
\hat{\beta}_{ORG} = \underset{\beta \in \hat{C}_{r}(\varepsilon)}{\text{argmax}} \: L_{\Psi}(\beta \mid \mathbf{X},y) + Pen(\hat{\tau}(\beta))
\end{equation}

\par If for our choice of $r$ we have $\hat{\beta}_{LR} \in \hat{C}_{r}(\varepsilon)$ and the distances between the set of breakpoints $\hat{\tau}(\hat{\beta}_{LR})$ are non-degenerate, then the global (unconstrained) optimality of $\hat{\beta}_{LR}$ ensures that it is also the optimal solution of the constrained ordinal problem. It may also serve as the optimal solution for the constrained and penalized ordinal problem, but that will depend on the selection of the aversion parameter. On the other hand, as we've seen in section \ref{section_lower_bounds_IRD}, once we have more than a single breakpoint we introduce limits of feasibility into the maximization problem and may discover that the solution to (\ref{optim_LR}) is no longer feasible ($\hat{\text{IRD}}_{r} (\hat{\beta}_{LR}) \geq \varepsilon$). For such cases we must define a new constrained optimization problem and look for a new optimal solution. We discuss an example for such a case in the following section.

\par There are two issues we leave outside the scope of this paper. First, although the consistency of constrained ML estimation has been explored in various contexts (for example for mixture models \cite{CML_hathaway_1985}), the consistency of the ML estimators under the specific constraint of $\hat{\text{IRD}}=0$ requires verification. Similarly, although in our description of the problem $\tau$ is a function of $\beta$ and the parameters of $\mathcal{F}$ (a result of the uniqueness demonstrated in appendix \ref{appendix_unique_tau}), it remains to be verified whether the consistency of the estimator $\hat{\beta}$ ensures the consistency of $\hat{\tau}(\hat{\beta})$. We expect the fact that for many cases $\hat{\tau}(\beta)$ has no analytical solution to further complicate this problem. 

\par Second, in order to measure our estimation errors we require a method for building right-sided confidence intervals for IRD based on the distribution of $\hat{\text{IRD}}$ for a given $\beta$. Since $\nu(\beta, \tau(\beta))$ is a ratio of CDF differences, the process of deriving the distribution of $\hat{\nu}(\beta, \hat{\tau}(\beta))$ from the distribution of $\hat{\eta}_{k}(\beta)$ would require several steps of approximation, primarily since $\hat{\tau}(\beta)$ the result of numeric estimation (even if the distribution of $\hat{\eta}_{k}(\beta)$ is known). Similar problems apply for non-parametric estimators, 
although we can see two possible approached for a solution. The first approach would be to use an equivalent definition of IRD as $\text{IRD}_{r}(\Psi, \tau) = \max_{i} \vert R_{i}(\Psi, \tau) - r_{i} \vert$ and try to prove a Glivenko-Cantelli \cite{vdV00} type theorem for conditional distributions which would describe the necessary conditions ensuring that:
\begin{equation} \label{eq_sup_cond_diff}
\sup_{x_{2} > x_{1}} \; \vert \hat{P}_{N}(Y = 1 \mid X \in (x_{1},x_{2}]) - P(Y = 1 \mid X \in (x_{1},x_{2}]) \vert \underset{N \rightarrow \infty}{\longrightarrow} 0
\end{equation}
where $\hat{P}_{N}$ is the empirical conditional probability estimator:
\begin{equation}
\hat{P}_{N}(Y = 1 \mid X \in (x_{1},x_{2}]) = \frac{\vert \lbrace j \; : \; \Psi(\mathbf{X}_{j}) \in (x_{1},x_{2}] \; \wedge \; y_{j} = 1 \rbrace \vert}{\vert \lbrace j \; : \; \Psi(\mathbf{X}_{j}) \in (x_{1},x_{2}]\rbrace \vert}
\end{equation}
Building on this result we can attempt to derive the asymptotic distribution of (\ref{eq_sup_cond_diff}) and try to construct test that will be the conditional equivalent of the Kolmogorov-Smirnov test \cite{vdV00}.
The second approach, which is less elegant but more plausible, would be to combine (\ref{eq_nu_def}) with the well known asymptotic behaviour of the empirical conditional distribution function: 
\begin{equation}
\hat{F}_{\eta_{k}(\beta)}^{(N_{k})}(t) = \frac{\vert \lbrace j \mid \Psi(\beta, \mathbf{X}_{j, \cdot}) \leq t, y_j = k \rbrace \vert}{N_{k}} \quad k=0,1
\end{equation}
where $N_{k} = \vert \lbrace j : y_{j} = k\rbrace \vert$. By the central limit theorem \cite{vdV00}, this estimator weakly converges to $F_{\eta_{k}(\beta)}$ pointwise:
\begin{equation}
\sqrt{N_{k}} \: \left( \hat{F}_{\eta_{k}}^{(N_{k})}(t) - F_{\eta_{k}}(t) \right) \underset{N_{k} \rightarrow \infty}{\longrightarrow} N \left( 0,F_{\eta_{k}}(t) \left( 1-F_{\eta_{k}}(t) \right) \right)
\end{equation}
However the fact that $\hat{\tau}(\beta)$ changes as a function of the sample (and is not a fixed quantile $t$) means that the points where $\hat{F}_{\eta_{k}}^{(N_{k})}$ is estimated change as a function of the sample ($N_{k}$), therefore the nature of the convergence and the resulting asymptotic distribution depend on the convergence of $\hat{\tau}(\beta) \rightarrow \tau(\beta)$. We would therefore need to find sufficient conditions for: 
\begin{equation} \begin{split}
&  \tau(\hat{\beta})^{(N_{k})} \rightarrow \tau(\beta) \: \Rightarrow \: \\ 
& \sqrt{N_{k}} \: \left( \hat{F}_{\eta_{k}}^{(N_{k})}(\tau(\hat{\beta})^{(N_{k})}) - F_{\eta_{k}}(\tau(\hat{\beta})^{(N_{k})}) \right) \underset{N_{k} \rightarrow \infty}{\longrightarrow} N \left( 0,F_{\eta_{k}}(\tau(\beta)) \left( 1-F_{\eta_{k}}(\tau(\beta)) \right) \right)
\end{split} \end{equation} 
The next step would be to use the strong uniform convergence of $\hat{F}_{\eta_{k}}^{(N_{k})} \rightarrow \hat{F}_{\eta_{k}}$ (as provided by the Glivenko-Cantelli theorem) to approximate $\hat{F}_{\eta_{k}}^{(N_{k})}(\hat{\tau}_{i}(\beta))$ as normally distributed $\mu = \hat{F}_{\eta_{k}}^{(N_{k})}(\hat{\tau}_{i}(\beta))$ and $\sigma^{2} = \hat{F}_{\eta_{k}}^{(N_{k})}(\hat{\tau}_{i}(\beta)) \left( 1 - \hat{F}_{\eta_{k}}^{(N_{k})}(\hat{\tau}_{i}(\beta)) \right) $. Combined with Donsker's theorem \cite{dudley1999}, we should be able to estimate the asymptotic distribution of the difference of two points of $\hat{F}_{\eta_{k}}^{(N_{k})}$ (which make both the numerator and the denominator of $\hat{nu}_{i})$ 
as normallt distributed with mean $\mu = \hat{F}_{\eta_{k}}(\hat{\tau}_{i} (\beta)) - \hat{F}_{\eta_{k}}(\hat{\tau}_{i-1} (\beta))$ and variance $\sigma^{2} = \left( \hat{F}_{\eta_{k}}(\hat{\tau}_{i} (\beta)) - \hat{F}_{\eta_{k}}(\hat{\tau}_{i-1} (\beta)) \right) \left( 1 - \hat{F}_{\eta_{k}}(\hat{\tau}_{i} (\beta)) - \hat{F}_{\eta_{k}}(\hat{\tau}_{i-1} (\beta)) \right) $. The final step would be to use work by Hinkley \cite{HINKLEY01121969}, which describes the distribution of a ratio of two non-correlated normal random variables, to approximate the distribution of $\hat{\nu}_{i}(\beta,\hat{\tau}(\beta))$ which we can use to estimate $P(\hat{\text{IRD}}_{r}(\beta) > \varepsilon)$. We note that the construction of such test would also mean that we can change the definition  of our feasibility set to $\hat{C}_{r}(\varepsilon, \alpha) = \lbrace \beta \in \mathbb{R}^{M} : \: P(\hat{\text{IRD}}_{r}(\beta) < \varepsilon) > 1-\alpha \rbrace$. We leave the details and proof of these ideas for future papers.

\section{Case study: Gaussian Logistic Regression \label{section_case_study}}

\par Logistic regression is one of the most studied classification methods in the scientific literature and has been widely applied in statistics, scientific research and industry. The name "logistic" for the function $f(x) = \frac{e^{x}}{1 + e^{x}}$ was originally coined by Verhulst as early as the 19'th century, but it was Cox \cite{cox_1969} who used it first in the context of binary data analysis. The concept of multinomial logistic regression was first suggested by Cox (1966) \cite{Cox_1966} and developed independently by Theil (1969) \cite{Theil_1969}. The link to ordered choice models (ordered logistic regression) was made by McFadden in his paper from 1973 \cite{McFadden_1973}. Cramer (2002) \cite{{RePEc:dgr:uvatin:20020119}} has a complete historical review.

\par Logistic regression belongs to the group of classification methods that estimate class membership probability rather than predict class membership. It is a special instance of a larger group of parametric models called generalized linear models (GLM \cite{GLM_1972}), which extends linear models by allowing the addition of a predefined link function $g$ that connects the linear model $\beta^{T} X$ ($\beta \in \mathbb{R}^{P}$) to the response variable $Y$ (meaning that $g^{-1} (Y) = \beta^{T} X$) and assuming that the distribution of $X$ is from an exponential family. In the case of logistic regression the link function is assumed to be the logistic function $g(\beta^{T} x) = \frac{e^{\beta^{T} x}}{1+e^{\beta^{T} x}}$, making the inverse function $g^{-1}(p) = \text{logit}(p) =log \left( \frac{p}{1-p} \right) $. The probability of belonging to the ``special class" (in the case of 2 classes) conditioned on the r.v $X$ is assumed to be:
\begin{equation} \label{eq_LR_model}
P_{\beta}(Y = 1 \mid X ) = \frac{e^{\beta^{T} X}}{1+e^{\beta^{T} X}}
\end{equation}
where the assumptions on $(X,Y)$ can be modified to match a wide variety of cases, for example to a non-random sampling scheme like case-control studies or semi-parametric models \cite{PRENTICE01121979}. 

\par For the purpose of demonstrating our method we assume that $\mathbf{X}$ is a $N \times P$ matrix representing $N$ i.i.d random samples from a $P$-dimensional multivariate normal distribution. 
Under this assumption the vector $y \in \lbrace 0,1 \rbrace^{N}$ of class memberships represents the result of $N$ independent Bernoulli random variables $\lbrace Y_{j} \rbrace_{j=1}^{T}$. Even under these assumptions, the ML problem does not have an analytical least squares solution, and is usually solved using numerical maximum likelihood algorithms. The log-likelihood function is:
\begin{equation}
l_{LR} (\beta \mid \mathbf{X},y) = \log(L_{LR}(\beta \mid \mathbf{X},y)) = \sum_{j=1}^{N}  y_{j} \beta^{T} \mathbf{X}_{j, \cdot}  - \sum_{j=1}^{N} \log \left( 1 + e^{ \beta^{T} \mathbf{X}_{j, \cdot} } \right)
\end{equation}
and the logistic regression (LR) maximum likelihood optimization problem is:
\begin{equation} \label{optim_LR}
\hat{\beta}_{LR} = \underset{\beta \in \mathbb{R}^{P}}{\text{argmax}} \; l_{LR}(\beta \mid \mathbf{X},y)
\end{equation}

\par As we have noted in section \ref{sec_IRD_est}, the parametric estimation of IRD requires several additional assumptions. We assume that the conditional distributions $X \mid Y=k$ ($k=0,1$) are also multi-variate normal, and in order to achieve SMLRP we assume equal conditional covariance. Using terms defined to construct (\ref{optim_constr_pen}) our assumptions on Gaussian logistic regression translate into the following:
\begin{equation}
\Psi(\beta, x) = \frac{\exp(\beta^{T}x)}{1 + \exp(\beta^{T}x)}, \quad X \mid Y=k  \sim MVN(\underline{\mu}_{k}, \Sigma)
\end{equation} 
As a result: 
\begin{equation} \label{eq_beta_X_norm_dist}
\text{logit}(\Psi(\beta, X)) \mid Y=k \: \sim \: N(\mu_{k}(\beta) = \beta^{T} \underline{\mu}_{k}, \sigma^{2}(\beta) = \beta^{T} \Sigma \beta)
\end{equation}
Conveniently, the ML estimators follow a similar pattern. For the construction of the known ML estimators of the distribution of $X$ we denote $\mathbf{X}^{(k)}$ as the matrix composed of all the lines of $\mathbf{X}$ for which $y_{j} = k$, $N_{1} = \sum_{j=1}^{N} y_{j}$, $N_{0} = N - \sum_{j=1}^{N} y_{j}$ and $\overline{\mathbf{X}}_{m}^{(k)} = \frac{\sum_{j=1}^{N_{k}} \mathbf{X}^{(k)}_{j,m}}{N_{k}}$ as the average of the $m$'th column of $\mathbf{X}^{(k)}$ ($m = 1, \dots, P$). The ML estimators are: 
\begin{equation} 
\hat{\underline{\mu}}_{k} = \overline{\mathbf{X}}^{(k)} = (\overline{\mathbf{X}}_{1}^{(k)}, \ldots, \overline{\mathbf{X}}_{P}^{(k)}) \quad (k = 0,1)
\end{equation}
and having assumed equal covariance we use the pooled covariance matrix estimator:
\begin{equation} \begin{split}
\hat{\Sigma}_{k} = & \frac{1}{N_{k}} \sum_{j=1}^{N_{k}} (\mathbf{X}^{(k)}_{j,\cdot} - \overline{\mathbf{X}}^{(k)})^{T} (\mathbf{X}^{(k)}_{j,\cdot} - \overline{\mathbf{X}}^{(k)}) \\
\hat{\Sigma} = & \frac{1}{N} ( N_{0} \hat{\Sigma}_{0} + N_{1} \hat{\Sigma}_{1}) \\ 
\end{split} \end{equation}

\par The assumption of multivariate-normal conditional distributions enables us to avoid the construction of an ancillary ML problem for each $\beta$ by using the known relationship 
between the ML estimators $(\hat{\underline{\mu}}_{0}, \hat{\underline{\mu}}_{1}, \hat{\Sigma})$ and the $\beta$-transformed ML estimators $(\hat{\mu_{k}}(\beta),  \hat{\sigma}(\beta))$: 
\begin{equation}
\hat{\mu_{k}}(\beta) = \beta^{T} \hat{\underline{\mu_{k}}}, \quad \hat{\sigma}(\beta) = \sqrt{\beta^{T} \hat{\Sigma} \beta}
\end{equation}

\par The strict monotonicity of $logit(p) = log \left( \frac{p}{1-p} \right) = \Psi^{-1}_{\beta}(p)$ in $p$ means that we can use the equality:
\begin{equation} \begin{split}
P & (\Psi(\beta, X) \in (\tau_{i-1}(\beta),\tau_{i}(\beta)] \mid Y=k) \\
& = P(\beta^{T} X \in (\text{logit}(\tau_{i-1}(\beta)),\text{logit}(\tau_{i}(\beta))]  \mid Y=k) \\
& = \Phi \left( \frac{\text{logit}(\tau_{i}(\beta)) - \mu_{k}(\beta)}{\sigma(\beta)} \right) - \Phi \left( \frac{\text{logit}(\tau_{i-1}(\beta)) - \mu_{k}(\beta)}{\sigma(\beta)} \right)
\end{split} \end{equation} 
to estimate $\nu_{i}(\beta)$ as: 
\begin{equation} 
\hat{\nu}_{i}(\beta) = 
\frac{
	\Phi \left( \frac{\text{logit}(\hat{\tau}_{i}(\beta)) -   \hat{\mu}_{1}(\beta)}{\hat{\sigma}(\beta)} \right) - 
	\Phi \left( \frac{\text{logit}(\hat{\tau}_{i-1}(\beta)) - \hat{\mu}_{1}(\beta)}{\hat{\sigma}(\beta)} \right)}
{
	\Phi \left( \frac{\text{logit}(\hat{\tau}_{i}(\beta)) - \hat{\mu}_{0}(\beta)}{\hat{\sigma}(\beta)} \right) - 
	\Phi \left( \frac{\text{logit}(\hat{\tau}_{i-1}(\beta)) - \hat{\mu}_{0}(\beta)}{\hat{\sigma}(\beta)} \right)} 
\end{equation}
Finally, we use the assumption of random sampling to estimate $\hat{p} = \frac{1}{N} \sum_{j=1}^{N} y_{j}$ and utilize (\ref{eq_nu_R}) to construct our parametric estimation of IRD for logistic regression.


\subsection{Example: The Wisconsin Diagnostic Breast Cancer \\ (WDBC) Dataset\label{section_example}}

\par In this section we bring an example of the sub-optimality of using one of the most commonly used classification methods - Logistic Regression (LR) - to solve a relatively simple ordinal problem. We then provide the Ordinal Risk-Group version of Logistic Regression (ORG-LR) solution to the problem and compare our results.

\par The dataset we used for this example is the extensively used Wisconsin Diagnostic Breast Cancer   \cite{10.1117/12.148698} dataset from the UCI Machine Learning Repository \cite{UCI},
which contains the analysis of cell nuclei from 556 patients using digitized images of fine needle aspirate (FNA) of extracted breast masses. Since we assumed a continuous $\mathcal{F}$ and equal variance we selected the following features for the construction of our models: texture, log area, smoothness, log compactness, log concave points, log symmetry and an intercept variable. The final result from the diagnosis ("Malignant" $N = 212$/ "Non-Malignant" $N = 344$) was used as the dependent variable for the logistic regression analysis and ordinal risk group analysis.
The code for this example was written in the R programming language \cite{cran} using internal optimization algorithms and the Augmented Lagrangian Adaptive Barrier Minimization Algorithm (the alabama library) with the constraint $\hat{\text{IRD}}_{r}(\beta)  < \varepsilon$ = 1e-07. 

\par The first set of risk levels we tested was $r_{1} = (10\%, 50\%, 90\%)$. The estimated IRD for the logistic regression solution $\hat{\beta}_{LR}$ and the matching (non-degenerate) set of breakpoints $\hat{\tau}(\hat{\beta}_{LR}) = (-2.6918, 9.1698)$ was slightly above our set feasibility threshold ($\hat{\text{IRD}}_{r_{1}} (\hat{\beta}_{LR})$ = 7.8776e-05). 
The second set of risk levels we tested was $r_{2} = (20\%, 50\%, 80\%)$. For this set the solution $\hat{\beta}_{LR}$ provided by the logistic regression was clearly infeasible: on the one hand the interval associated to the 50\% risk level was clearly degenerate ($ \hat{\tau}_{2}(\hat{\beta}_{LR}) - \hat{\tau}_{1}(\hat{\beta}_{LR})  <$ 1.1e-06) and the estimated IRD was high above our set threshold ($\hat{\text{IRD}}_{r_{2}}(\hat{\beta}_{LR}) = 0.0014 > \varepsilon$). We proceeded to construct a constrained maximum likelihood problem for both $r_{1}, r_{2}$ as described in section \ref{section_case_study}. For $r_{1}$ an unconstrained problem was sufficient, with IRD $<$ 1e-07 and a non-degenerate $\hat{\tau}(\hat{\beta})$. For $r_{2}$, the unpenalized ordinal risk-group problem produced degenerate solutions, so we added the penalty function $Pen(\beta, \tau) = (\frac{\max \vert \tau_{i} - \tau_{i-1} \vert}{\beta^{T}(\hat{\mu}_{1} - \hat{\mu}_{0})} -1)^{2}$, which was designed to balance between the distance between the breakpoints $\tau_{1}, \tau_{2}$ and the distance between estimated class means, and selected a penalty coefficient $\gamma = 10$. Since in many cases the optimization algorithm converged to a local minimum we randomly sampled 25,000 starting points for each set of risk categories we tested. The estimates for logistic regression (LR) and ordinal risk-group logistic regression (ORG-LR) for both  $r_{1}, r_{2}$ are summarized in table \ref{table_LR_ORG}.
\begin{table}
\begin{center} \begin{tabular}{l|ll|ll}
& \multicolumn{2}{c|}{$r_{1}$} & \multicolumn{2}{c}{$r_{2}$} \\
\hline
\hline
 & \multicolumn{1}{c}{LR} & ORG-LR & \multicolumn{1}{c}{LR} & ORG-LR \\
\hline
Intercept & -87.6641 & 12.3928 & -87.6641 & -3.1977 \\
Texture & 0.2864 & 0.1894 & 0.2864 & 0.1844\\
log(Area)  & 11.9706 & 0.8999 & 11.9706 & -0.3573 \\
Smoothness & 67.7342 & -25.2575 & 67.7342 & 4.0826 \\
log(Compactness) & -1.8107 & -3.9205 & -1.8107 & 0.3070 \\
log(Concave Points) & 3.6698 & 11.8320 & 3.6698 & 0.1162 \\
log(Symmetry) & 3.2556 & 0.7888 & 3.2556 & -0.3529 \\
\hline
log-Likelihood &  -45.9909 & -94.8944 & -45.9909 & -311.019 \\
\hline
$\hat{\tau}_{1}$ & -2.6918 & -4.4029 & 3.117868 & -0.7088  \\
$\hat{\tau}_{2}$ & 9.1698 & 6.8524 & 3.117869 & 1.3376 \\
\hline
$P(Y = 1 \mid \Psi(\hat{\beta},X) \leq \tau_{1})$ & 9.2925\% & 9.9751\% & 16.7680\% & 19.9857\% \\
$P(Y = 1 \mid \Psi(\hat{\beta},X) \in (\tau_{1}, \tau_{2}])$ & 50.3289\% & 50.0130\% & 51.5483\% & 50.0116\% \\
$P(Y = 1 \mid \Psi(\hat{\beta},X) > \tau_{2})$ & 89.5768\% & 89.9870\% & 78.7930\% & 79.9948\% \\	
\hline
$\hat{\text{IRD}}$ & 7.88e-05 & 9.69e-08 & 0.0014 & 3.66e-08 \\ 
\end{tabular} \end{center}
\caption[LR and ORG-LR maximum likelihood estimators for $r_{1},r_{2}$]{Maximum likelihood (ML) estimators of coefficients, log-likelihood, optimal breakpoints $\tau$, model predicted probabilities (assuming multivariate normal distribution) and IRD estimates for unconstrained logistic regression and ordinal risk-group logistic regression (ORG-LR) for $r_{1} = (10\%, 50\%, 90\%)$, $r_{2} = (20\%, 50\%, 80\%)$ \label{table_LR_ORG}}
\end{table}

\par The differences between the two methods can be further illustrated by looking at the distributions of the logit-transformed predicted probabilities $\lbrace logit(\hat{Y_{i}}) \rbrace_{i=1}^{N}$ of both methods. Figure \ref{figure_LR_vs_ORG-LR_10-50-90} illustrates the logistic regression solution for $r_{1}$ (top graph) and the ordinal risk-group logistic regression solution for $r_{1}$ (bottom graph), and figure \ref{figure_LR_vs_ORG-LR_20-50-80} illustrate the same results for $r_{2}$. A comparison of the two graphs in each figure shows that for both sets of risk levels, the ORG-LR solution compromises the quality of separation between the two classes in order to achieve feasibility, and in the more extreme case of $r_{2}$ reduces separation dramatically in order to avoid degenerate solutions.
\begin{figure}[h]
\center
\includegraphics[scale=0.9]{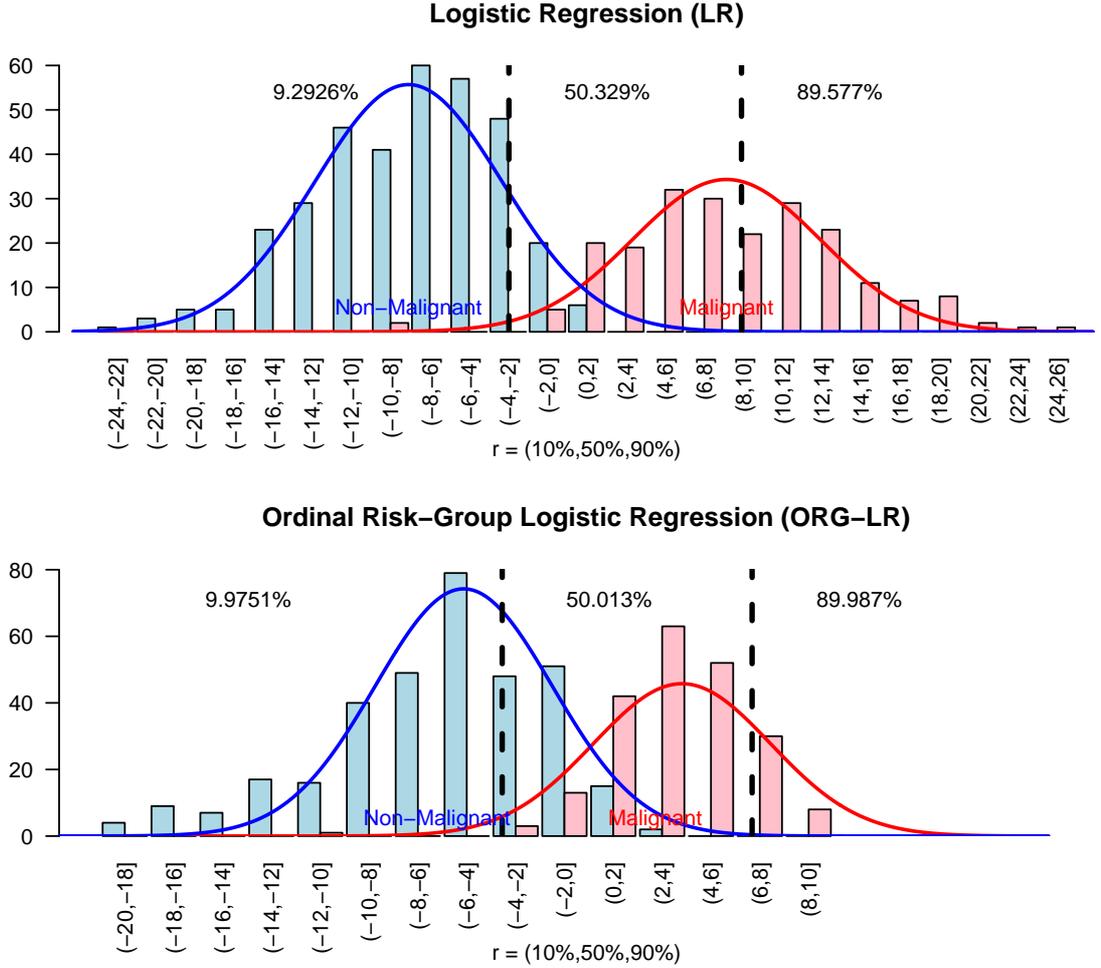}
\caption[The WDBC dataset: logit-transformed LR and ORG-LR predictions for $r_{1}$]{\label{figure_LR_vs_ORG-LR_10-50-90} Logit-transformed predictions and separation between the malignant and non-malignant classes of logistic gegression (LR) (top graph) and the Ordinal Risk-Group Logistic Regression (ORG-LR) (bottom graph) for $r_{1}$. The black dotted lines mark the matching sets of breakpoints $\hat{\tau}(\hat{\beta}_{LR})$ and $\hat{\tau}(\hat{\beta}_{ORG-LR})$.}
\end{figure}
\begin{figure}[h]
\center
\includegraphics[scale=0.9]{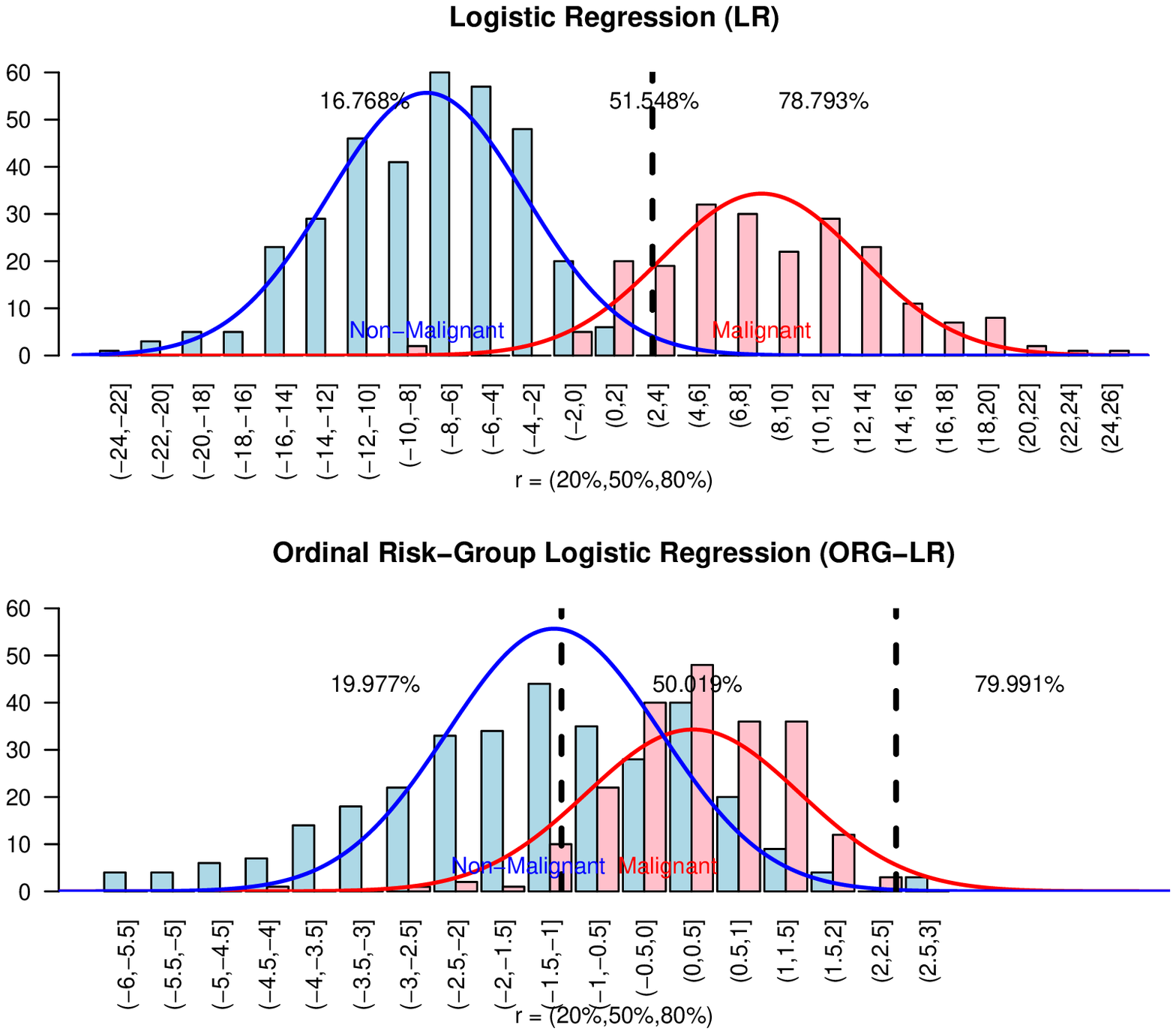}
\caption[The WDBC dataset: logit-transformed LR and ORG-LR predictions for $r_{2}$]{\label{figure_LR_vs_ORG-LR_20-50-80} Logit-transformed predictions and separation between the malignant and non-malignant classes of logistic gegression (LR) (top graph) and the Ordinal Risk-Group Logistic Regression (ORG-LR) (bottom graph) for $r_{2}$. The black dotted lines mark the matching sets of breakpoints $\hat{\tau}(\hat{\beta}_{LR})$ and $\hat{\tau}(\hat{\beta}_{ORG-LR})$.}
\end{figure}

\par In order to validate the results of our new method and compare them to the performance of the logistic regression solution we performed a cross validation study. We randomly divided the dataset into two groups: 90\% of the patients were randomly sampled as a training set, from which a logistic regression model and an ordinal risk-group model were constructed, and the remaining 10\% were used as a test set for the models. We repeated the process with 25,000 random samples and calculated the percentage of "Melignant" cases found in each predicted risk group for each of the models. The results of the implementation of the training models on the test sets for $r_{1} = (10\%, 50\%, 90\%)$ were $(0.7173\%, 74.6978\%, 100\%)$ for logistic regression (IRD = 0.07962) and $(3.8804\%, 57.4549\%, 99.9329\%)$ for ordinal risk-group logistic regression (IRD = 0.01917). The cross validation results for $r_{2} = (20\%,50\%,80\%)$ using ORG-LR were $(15.7349\%, 52.9507\%, 84.9715\%)$ (IRD = 0.0052). Since the logistic regression solution was degenerate for $r_{2}$ we were unable test it with cross validation. 

\par The comparison of the cross validation results from the two methods for $r_{1}$ shows that although both models did not perform wery well, the ORG-LR solution outperformed the logistic regression solution (IRD is approx. 4 times smaller), in spite of the fact that differences in IRD between the two models do not seem significant. We estimate that one of the reasons for the of the high absolute deviance in IRD of all models and risk levels we tested is that the data is not exactly normally distributed (as evident in figures \ref{figure_LR_vs_ORG-LR_10-50-90},\ref{figure_LR_vs_ORG-LR_20-50-80}). In addition, specifically for ORG-LR, we suspect that the generic algorithms we used for the ordinal risk-group maximum likelihood constrained optimization failed to converge to the real global minimum in some of the cross-validation iterations. Verifying this hypothesis would require either use of more specific optimization algorithms (for example by analytically calculating the derivatives of the IRD constraint) or by a very large scale simulation study that would require billions of iterations. Both approaches are outside the scope of this paper and we leave them for future studies.

\section{Conclusions}

\par The exact estimation of the conditional risk function is an important part of practical and theoretical research. However the practical application of this information is very often in the form of a finite and small set of resulting actions. Although conditional risk quantiles provide valuable information, we ultimately want to know the risk associated with adjacent non-overlapping intervals in order to create distinct ordinal risk groups. As we have demonstrated in section \ref{section_cond_precentiles}, quantile regression is not useful for that purpose. Furthermore, section \ref{section_lower_bounds_IRD} demonstrates that the practice of dividing post-hoc the continuous estimate of conditional risk into intervals ignores the limitations introduced by the lower bounds on IRD and may produce sub-optimal or degenerate solutions.

\par Our formulation of the optimization problem, as presented in section \ref{section_ORGC}, reflects our understanding that while the model's ability to separate the classes remains the key issue, we must introduce both a new constraint and a penalty function in order to achieve two additional objectives: an accurate risk distribution and a usable partition scheme. While IRD represents an absolute measure of the model's quality and must be a constraint on the optimal solution, the "softer" requirement on minimal interval length should allow for flexibility in application. We believe that a penalty function enables better control and adaptation through the choice of function and the aversion parameter.

\par Finally, we wish to emphasize the implications of the most counter intuitive result of this paper - the existence of limitations on certain risk structures (the vector $r$) in the form of lower bounds on the error rate IRD (equation \ref{eq_IRD_def}).  Although most of the examples we described are linear or logistic models with Gaussian conditional distributions, the existence of lower bounds holds for any continuous risk estimator. A re-evaluation of the optimal properties of such estimators in the context of risk discretization is therefore required. We leave the specifics of applying these ideas to other classification methods as well as proofs of consistency and the construction of confidence intervals to future studies.

\addcontentsline{toc}{section}{Reference}
\bibliographystyle{plain}
\bibliography{arxiv}

\appendix
\section{Appendix: Equivalence of strict monotonicity of CDF ratios over intervals and the Strict Monotone Likelihood Ratio Property \label{appendix_MLRP_cdf_ratio_mono}}

\par Let $X$ be a continuous $P$-dimensional random vector and $Y \in \lbrace 0,1 \rbrace$ a Bernoulli random variable representing class membership. Assume that the conditional distributions $X \mid Y=k$ ($k = 0,1$) are also continuous and the conditional densities $f_{X \mid Y = k}$ are finite. Let $\Psi : \mathbb{R}^{P} \rightarrow \mathbb{R}$ be a finite continuous risk predictor. For a single observation $X$, the \emph{likelihood ratio} of the risk predictor $\Psi$ between the two alternatives represented by $Y$ is defined as the ratio of the conditional densities:
\begin{equation*}
\Lambda_{\Psi}(x) = \frac{P(\Psi(X) = x \mid Y=1)}{P(\Psi(X) = x \mid Y=0)} = \frac{f_{\Psi(X) \mid Y=1}(x)}{f_{\Psi(X) \mid Y=0}(x)}
\end{equation*}

\par The \emph{Strictly Monotone Likelihood Ratio Property} (SMLRP) demands that for a given $X, Y, \Psi$ the ratio $\Lambda_{\Psi}(x)$ is a strictly monotone function of $x$. It is worth noting that by using Bayes theorem we can show that demanding SMLRP is equivalent to demanding strict monotonicity of $P(Y=1 \mid \Psi(X)=x)$ in $x$:
\begin{equation}
P(Y=1 \mid \Psi(X)=x) = \frac{p f_{1}(x)}{p f_{1}(x)+(1-p) f_{0}(x)} = \frac{1}{1 + \frac{1-p}{p} \frac{1}{\Lambda_{\Psi}(x)}}
\end{equation}
where $p = P(Y=1)$ and $f_{k} (x) = f_{\Psi(X) \mid Y=k} (x)$ are
the density functions of the conditional distributions. This
equivalence means that in terms of conditional probability, SMLRP is
equivalent to strict \emph{pointwise} monotonicity in the condition,
in contrast to the requirement of monotonicity over right-expanding
intervals in section \ref{section_lower_bounds_IRD}, which we
defined as the strict  monotonicity of $R(\beta, \tau)_i = P(Y=1
\mid \Psi(X) \in (\tau_{i-1}, \tau_{i}])$ in $\tau_{i}$ for any
$\tau_{i-1}$ while $\tau_{i} > \tau_{i-1}$.

\begin{theorem} \label{th_SMLRP_mono_R_equiv}
SMLRP $\Leftrightarrow$ $\forall \tau_{i-1}$, $\forall \tau_{i} >
\tau_{i-1}$ $R(\beta, \tau)_i$ is strictly increasing in $\tau_{i}$.
\end{theorem}

\begin{proof} Using our previous definition of $R(\beta, \tau)_i = P(Y=1 \mid \Psi(X) \in (\tau_{i-1}, \tau_{i}])$ and Bayes theorem we can represent:
\begin{equation} \label{eq_R_gamma_equiv}
R(\Psi, \tau)_i = \frac{p (F_{1}(\tau_{i}) - F_{1}(\tau_{i-1}))}{p (F_{1}(\tau_{i}) - F_{1}(\tau_{i-1})) + (1-p) (F_{0}(\tau_{i}) - F_{0}(\tau_{i-1}))} = \frac{1}{1 + \frac{1-p}{p} \frac{1}{\gamma_{\Psi}(\tau_{i-1},\tau_{i})}}
\end{equation}
where
\begin{equation} \label{eq_gamma_def}
\gamma_{\Psi}(\tau_{i-1},\tau_{i}) = \frac{F_{1}(\tau_{i}) - F_{1}(\tau_{i-1})}{F_{0}(\tau_{i}) - F_{0}(\tau_{i-1})}
\end{equation}
and $F_{k} (x) = F_{\Psi(X) \mid Y=k} (x)$ are the cumulative distribution functions (CDF) of the conditional distributions. The strict monotonicity of $R_{i}(\Psi, \tau)$ in $\tau_{i}$ for any $\tau_{i-1} < \tau_{i}$ is therefore equivalent to the strict monotonicity of $\gamma(c,x)$ in $x$ for any $c, \: x > c$.

\par In addition, for two positive, finite, strictly increasing and once differentiable functions $g,h$ the following equivalence holds:
\begin{equation}
\frac{g(x)}{h(x)} \text{ is strictly increasing} \: \Leftrightarrow \: \frac{g'(x) h(x) - g(x) h'(x)}{h^{2} (x)} > 0 \: \Leftrightarrow \: \frac{g(x)}{h(x)} < \frac{g'(x)}{h'(x)}
\label{eq_ratio_mono}
\end{equation}

Since $F_{0}, F_{1}$ meet these requirements, then by (\ref{eq_R_gamma_equiv}) the strict monotonicity of both $\gamma_{\Psi}(c,x)$ and $R(\Psi, \tau)_i$ is equivalent to following condition:
\begin{equation} \label{eq_lambda_gamma_ineq}
\gamma_{\Psi}(c,x) = \frac{F_{1}(x) - F_{1}(c)}{F_{0}(x) - F_{0}(c)} < \frac{\frac{d}{dx}(F_{1}(x) - F_{1}(c))}{\frac{d}{dx}(F_{0}(x) - F_{0}(c))} = \frac{f_{1}(x)}{f_{0}(x)} = \Lambda_{\Psi}(x) \quad \forall c<x
\end{equation}

It is therefore sufficient to show that under the above assumptions of continuity and finiteness that the following equivalence holds:
\begin{equation}
SMLRP \: \Leftrightarrow \: \gamma_{\Psi}(c,x) < \Lambda_{\Psi}(x) \quad \forall c,c<x
\end{equation}

\par \emph{Step 1:} $SMLRP \: \Rightarrow \: \gamma_{\Psi}(c,x) < \Lambda_{\Psi}(x) \quad \forall c<x$ \\
Under SMLRP:
\begin{equation} \label{MLRP_mono_ratio}
\forall x_{1} > x_{0} \quad \frac{f_{1} (x_{1})}{f_{0} (x_{1})} > \frac{f_{1} (x_{0})}{f_{0} (x_{0})} \: \Leftrightarrow \: f_{1} (x_{1}) f_{0} (x_{0}) > f_{1} (x_{0}) f_{0} (x_{1})
\end{equation}
The equivalence holds since $f_{0}, f_{1}$ are strictly positive, continuous and finite. Integrating on $x_{0}$ over the interval $[c,x_{1}]$ we have:
\begin{equation} \label{eq_MLRP_monotone_Psi} \begin{split}
& \int_{c}^{x_{1}} f_{1} (x_{1}) f_{0} (x_{0}) d x_{0} > \int_{c}^{x_{1}} f_{1} (x_{0}) f_{0} (x_{1}) d x_{0} \\
\Leftrightarrow \quad & f_{1} (x_{1}) (F_{0} (x_{1}) - F_{0} (c)) >
f_{0} (x_{1}) (F_{1} (x_{1}) - F_{1} (c)) \: \Leftrightarrow \:
\gamma_{\Psi}(c,x_{1}) < \Lambda_{\Psi}(x_{1}) \end{split}
\end{equation}
and this holds for any $x_{1} \in \mathbb{R}$. In
addition setting $c = -\infty$ when integrating maintains the strict
inequalities of (\ref{eq_MLRP_monotone_Psi}), and therefore SMLRP
also ensures strict monotonicity of $\gamma_{\Psi}(-\infty, x) =
F_{1}(x) / F_{0}(x)$ and the equivalent monotonicity of $P(Y=1 \mid
\Psi(X) < x)$ in $x$.

\par \emph{Step 2:} $ \forall c,c<x \quad \gamma_{\Psi}(c,x) < \Lambda_{\Psi}(x) \: \Rightarrow \: SMLRP$

\par Under our assumptions:
\begin{equation}
\forall x_{1} > x_{0} > c \quad \gamma(c,x_{1}) = \frac{\int_{c}^{x_{1}} f_{1} (x) dx}{\int_{c}^{x_{1}} f_{0} (x) dx} > \frac{\int_{c}^{x_{0}} f_{1} (x) dx}{\int_{c}^{x_{0}} f_{0} (x) dx} = \gamma(c,x_{0})
\end{equation}
Assuming all functions are continuous and finite we can take $c \rightarrow x_{0}$:
\begin{equation} \label{eq_gamma_ineq_rhs}
\gamma(x_{0},x_{1}) = \frac{\int_{x_{0}}^{x_{1}} f_{1} (x) dx}{\int_{x_{0}}^{x_{1}} f_{0} (x) dx} > \frac{f_{1} (x_{0})}{f_{0} (x_{0})}
\end{equation}
This inequality is strict since for any $x_{1} > x_{0}$ there exists
$\varepsilon = \frac{x_{1} - x_{0}}{2}>0$ such that:
\begin{equation}
\gamma(x_{0},x_{1}) = \frac{\int_{x_{0}}^{x_{1}} f_{1} (x) dx}{\int_{x_{0}}^{x_{1}} f_{0} (x) dx} >  \frac{\int_{x_{0}}^{x_{0}+\varepsilon} f_{1} (x) dx}{\int_{x_{0}}^{x_{0}+\varepsilon} f_{0} (x) dx} = \gamma(x_{0},x_{0}+\varepsilon) \geq \frac{f_{1} (x_{0})}{f_{0} (x_{0})}
\end{equation}
On the other hand taking $c \rightarrow x_{1}$ (using the same
considerations and utilizing the fact that for $b>a$, $\int_{b}^{a}
f(x)dx = - \int_{a}^{b} f(x)dx$) and combining with
(\ref{eq_gamma_ineq_rhs}) we have:
\begin{equation}  \label{eq_gamma_ineq_lhs}
\frac{f_{1} (x_{1})}{f_{0} (x_{1})} > \frac{\int_{x_{0}}^{x_{1}} f_{1} (x) dx}{\int_{x_{0}}^{x_{1}} f_{0} (x) dx} = \gamma(x_{0},x_{1}) > \frac{f_{1} (x_{0})}{f_{0} (x_{0})}
\end{equation} \end{proof}

\section{Appendix: Uniqueness of $\tau$ under the Strict Monotone Likelihood Ratio Property \label{appendix_unique_tau}}

\par The risk estimation methods mentioned in this paper typically deal only with the optimal estimation of $\Psi$ (and the breakpoints $\tau$ are defined post-hoc). Introducing the set of breakpoints $\tau$ as an integral part of the definition of IRD increases in the number of parameters that must be estimated simultaneously, resulting in a more complicated parameter space (for example we require $\tau_{i-1} < \tau_{i}$). Although the increase in the number of estimated parameters should not be significant (in practical scenarios we expect $T \leq 10$) the result nonetheless would be longer running times for the optimization algorithms. Before we proceed any further it would be useful to identify sufficient conditions for the uniqueness of $\tau$ for a given $\Psi$:

\begin{lemma} \label{lemma_tau_unique}
If for a given $\Psi$ the likelihood ratio $\Lambda_{\Psi}(x) = \frac{f_{\Psi(X) \mid Y=1}(x)}{f_{\Psi(X) \mid Y=0}(x)}$ satisfies the strict monotone likelihood ratio property (SMLRP), then if there exists $\tau_{\Psi}$ such that $\text{IRD}_{r}(\Psi, \tau_{\Psi}) =0$ it is unique.
\end{lemma}
\begin{proof} If $\Lambda_{\Psi}(x)$ satisfies SMLRP, then by theorem \ref{th_SMLRP_mono_R_equiv} (appendix \ref{appendix_MLRP_cdf_ratio_mono}) $R$ is strictly monotone is $\tau_{i}$. The rest is by induction: strict monotonicity of $R_1$ means that if there exists $\tau_{1}$ which satisfies $R(\Psi, \tau)_1 = r_1$, then it is unique. Fixing $\tau_{i-1}$, if (\ref{eq_lower_bound_on_r}) holds (meaning that $\tau_{i}$ is "feasible"), then again by strict monotonicity, if there exists $\tau_{i}$ that satisfies  $R(\Psi, \tau)_i = r_{i}$, then it is unique.
\par Therefore if (\ref{eq_lower_bound_on_r}) holds for all $i$ then only a single $\tau$ satisfies $\text{IRD}_{r}(\Psi, \tau) = 0$.\end{proof}

\begin{corollary} Under SMLRP we can denote $\tau = \tau(\Psi)$ and define IRD using $\Psi$ alone:
\begin{equation} \label{eq_defs_Psi_only}
R_{i}(\Psi) = P(Y=1 \mid \Psi(X) \in (\tau_{i-1}(\Psi), \tau_{i}(\Psi)]), \quad \text{IRD}_{r}(\Psi) = \Vert R(\Psi) - r \Vert
\end{equation} \end{corollary}

\end{document}